\documentclass{article}

\usepackage[accepted,nohyperref]{eiml_icml2026}
\usepackage{amsmath,amssymb,amsthm}
\usepackage{booktabs}
\usepackage{graphicx}
\usepackage{multirow}
\usepackage{microtype}
\usepackage{xurl}
\usepackage{enumitem}
\usepackage{placeins}
\usepackage{xcolor}
\usepackage{fancyvrb}
\graphicspath{{./}{figures/}}

\newtheorem{proposition}{Proposition}

\newcommand{\E}{\mathbb{E}}

\newcommand{\Cov}{\mathrm{Cov}}
\newcommand{\diag}{\mathrm{Diag}}
\newcommand{\softmax}{\mathrm{softmax}}
\icmltitlerunning{Same Target, Different Basins: Hard vs. Soft Labels for Annotator Distributions}
\begin{document}

\twocolumn[
\icmltitle{Same Target, Different Basins: \\Hard vs. Soft Labels for Annotator Distributions}

\begin{icmlauthorlist}
\icmlauthor{Mirerfan Gheibi}{ind}
\icmlauthor{Gashin Ghazizadeh}{ind}
\end{icmlauthorlist}

\icmlaffiliation{ind}{Independent Researcher}

\icmlcorrespondingauthor{Mirerfan Gheibi}{m.erfan.gh@gmail.com}
\icmlcorrespondingauthor{Gashin Ghazizadeh}{g.gashin@gmail.com}

\icmlkeywords{epistemic alignment, annotator disagreement, soft labels, hard-label delivery, stochastic label sampling, loss landscape}

\vskip 0.3in
]

\printAffiliationsAndNotice{}

\begin{abstract}
When annotators disagree, that disagreement can reflect epistemic uncertainty rather than simple label noise. We study hard-label delivery as an alternative to the usual choices of collapsing votes to a single label or training directly on the empirical soft-label distribution. We focus on two primary hard-label methods: multipass, which cycles through observed votes while keeping the dataset size fixed, and stochastic label sampling (SLS), which samples one label per example at the start of each epoch. On CIFAR-10H, we find that when only a small number of annotations per example is available, hard-label delivery improves over soft-label training, with larger improvements where the sparse empirical target is farther from the full annotator distribution. When full annotator distributions are available, both hard-label methods match soft-label training. We use deterministic control as an ablation of multipass and shuffled SLS as a control that breaks the example-to-distribution match. We also show that SLS and soft-label cross-entropy optimize the same expected objective. Hard-label delivery also converges to flatter basins, with supporting descriptive evidence from OOD detection on SVHN and CIFAR-100. Overall, these results suggest that multipass is a strong practical default when raw vote counts are available, while SLS offers a lightweight alternative that remains competitive when only a few votes per example are available and matches soft-label training when full annotator distributions are available.
\end{abstract}

\section{Introduction}
In supervised learning with multiple annotations per example, a common simplification is to collapse those annotations to a single label. That convention is convenient, but it can discard useful information when annotators disagree. Prior work has argued that disagreement should not always be treated as mere annotation error or collapsed away, especially when it reflects ambiguity, subjectivity, or multiple plausible interpretations \citep{peterson2019human,uma2021learning,plank2022problem}. In this paper, we ask a narrower question: if the goal is to learn from the annotator distribution, does that distribution need to remain inside the loss at every step, or can it instead be delivered through hard labels during training?

We study a family of hard-label delivery methods that keep the per-example annotator target fixed while changing only how it is presented to the optimizer. Our primary method is multipass, which cycles through the observed votes for each example while keeping the dataset at $N$ examples. We also study stochastic label sampling (SLS), which draws one label per example at the start of each epoch and then trains with ordinary hard-label cross-entropy. To understand what is driving any observed differences, we include two controls. One repeats the multipass procedure with a different fixed vote order, testing whether traversal order matters. The other shuffles annotator distributions across examples before SLS, preserving overall disagreement statistics while breaking the link between each example and its own distribution.

Prior work already shows the value of learning from annotator distributions \citep{peterson2019human}. Our goal is different: we hold the target distribution fixed and ask what changes when only the delivery format changes. SLS is especially useful analytically because Proposition~\ref{prop:main} shows that SLS and soft-label cross-entropy optimize the same expected objective. This makes it possible to separate differences in optimization path from differences in target. The broader family view also matters in practice: when raw vote counts are available, multipass is the hard-label method of primary interest, while SLS plays the same role when only per-example probabilities are available.

The experiments show a clear regime split. When only a few annotations per example are available, hard-label delivery is directionally better than soft-label training on soft NLL, with larger gains where the sparse empirical target is farther from the full annotator distribution. When full annotator distributions are available, multipass and SLS match soft-label training on the primary endpoint metrics. The two controls clarify why: changing traversal order does not materially change the result, but breaking the pairing between examples and their own annotator distributions does. Hard-label delivery also converges to flatter basins, with supporting descriptive gains in OOD uncertainty analyses.

Our contributions are:
\begin{itemize}[leftmargin=1.2em]
  \item \textbf{A controlled comparison of delivery format.} We compare hard-label delivery with soft-label training while keeping the annotator target fixed. On CIFAR-10H, hard-label delivery is directionally better than soft-label training across the sparse annotator-count sweep and matches soft-label training in the full-distribution regime. The controls show that preserving the example-to-distribution match is essential.
  \item \textbf{Expected-objective equivalence for SLS.} Proposition~\ref{prop:main} shows that SLS and soft-label cross-entropy optimize the same expected objective, and gives the corresponding closed-form gradient covariance under label sampling.
\end{itemize}

\section{Related Work}
\label{sec:related}

\paragraph{Learning from annotator disagreement.}
Prior work treats annotator disagreement as information rather than noise to be collapsed away. Classical crowd models estimate latent labels together with annotator reliabilities \citep{dawid1979maximum,raykar2010learning,rodrigues2018deep}, while later work argues that disagreement can reflect ambiguity, subjectivity, or multiple plausible interpretations \citep{aroyo2015truth,uma2021learning,plank2022problem}. In vision, CIFAR-10H \citep{peterson2019human} made this concrete by releasing human label distributions for image classification and showing the value of training against them; \citet{battleday2020capturing} interpret those distributions as perceptual ambiguity. Related ideas appear in NLP as well, including disagreement-aware modeling \citep{mostafazadehdavani2022dealing}, perspectivist approaches \citep{frenda2025perspectivist,xu2026beyond}, and benchmarks such as ChaosNLI \citep{nie2020what} and LeWiDi \citep{leonardelli2023semeval}. Our paper takes the annotator distribution as given and asks a narrower question: how should that target be delivered during training?

\paragraph{Hard-label delivery and related baselines.}
The closest prior work to our setting comes from two places. \citet{peterson2019human} already explore sampled hard targets from human label distributions, so our contribution is not the sampling idea itself, but a controlled comparison of hard-label delivery against soft-label training with the annotator target held fixed. \citet{sheng2008get} study repeated labeling and include a strategy that expands multiple labels into repeated training examples; multipass can be viewed as a cardinality-preserving version of that idea. More broadly, soft-label training is a form of distributional supervision, as in distillation \citep{hinton2015distilling}, although annotator distributions encode task ambiguity rather than teacher confidence. SLS is also related to DisturbLabel \citep{xie2016disturblabel}, but the sampled labels here come from the empirical annotator distribution rather than from uniform label noise. We include label smoothing \citep{szegedy2016rethinking,muller2019when} and mixup \citep{zhang2018mixup} as familiar baselines, although they address a different question.

\paragraph{Related but distinct literatures.}
Our setting is adjacent to, but distinct from, standard noisy-label learning. There, the goal is usually to recover or correct an underlying clean label under corruption \citep{natarajan2013learning,frenay2014classification,song2023learning}. Here, the annotator distribution is itself the target. This distinction also matters for the geometry results. Prior work links optimization noise and label noise to flatter solutions \citep{hochreiter1997flat,keskar2017largebatch,smith2021origin,damian2021label}, with more recent work emphasizing the role of noise covariance and its alignment with the Hessian \citep{haochen2021shape,wu2022alignment}. We use that literature to interpret basin differences between hard and soft delivery, but geometry is a supporting characterization rather than the paper's main empirical claim.

\section{Method}
\label{sec:method}

Let $\mathcal{D}=\{(x_i,p_i)\}_{i=1}^N$ be a dataset where $p_i\in\Delta^{K-1}$ is the empirical annotator distribution for example $x_i$. Let $z_\theta(x)\in\mathbb{R}^K$ denote the logits and $q_\theta(x)=\softmax(z_\theta(x))$ the model prediction. We compare soft-label training with hard-label delivery methods that keep the same per-example target. Our primary hard-label methods are multipass and stochastic label sampling (SLS), with deterministic control and shuffled SLS as controls.

\paragraph{Soft-label training.}
Soft-label training uses the full empirical annotator distribution in the loss for each example. The objective is the cross-entropy between $p_i$ and $q_\theta(x_i)$:
\begin{align}
\mathcal{L}_{\mathrm{soft}}(\theta)
&=\sum_{i=1}^N H\bigl(p_i,q_\theta(x_i)\bigr) \\
&= -\sum_{i=1}^N \sum_{k=1}^K p_{ik}\log q_{\theta,k}(x_i).
\end{align}

\paragraph{Multipass.}
When raw vote counts are available, multipass expands each example's count vector into its multiset of observed labels, shuffles that sequence once per example with a fixed seed, and cycles through it across epochs via $\texttt{epoch} \bmod \texttt{len(sequence)}$. Each epoch still contains the same $N$ examples; only the delivered hard label changes. This preserves dataset cardinality. For an example with $m_i$ observed votes, the delivered label follows a length-$m_i$ cycle through its shuffled vote multiset, so each observed vote appears once before the sequence repeats.

\paragraph{Stochastic label sampling.}
SLS is the lightweight hard-label alternative when only per-example probabilities are available. At the start of epoch $t$, each example $x_i$ is assigned a hard label sampled from its empirical annotator distribution:
\begin{equation}
    y_i^{(t)} \sim \mathrm{Categorical}(p_i).
\end{equation}
Training then proceeds for that epoch with ordinary hard-label cross-entropy on the sampled labels.

SLS is especially useful here because it enables a clean comparison with soft-label training while keeping the expected objective fixed.

\begin{proposition}[Expected objective and gradient covariance]\label{prop:main}
For a fixed example $x$ with annotator distribution $p$ and predictive probabilities $q=\softmax(z)$,
\begin{align}
\E_{y\sim p}\bigl[-\log q_y\bigr]
&= -\sum_{k=1}^K p_k \log q_k = H(p,q), \\
\E_{y\sim p}\bigl[\nabla_z \ell(z,y)\bigr]
&= q-p, \\
\Cov_{y\sim p}\bigl[\nabla_z \ell(z,y)\bigr]
&= \diag(p)-pp^\top.
\end{align}
If $J_\theta(x)=\partial z_\theta(x)/\partial \theta$ is the logit Jacobian, then
\begin{equation}
\Cov\bigl[\nabla_\theta \ell(z_\theta(x),y)\bigr]
= J_\theta(x)^\top\bigl(\diag(p)-pp^\top\bigr)J_\theta(x).
\end{equation}
\end{proposition}

\begin{proof}
The first identity is linearity of expectation. For hard-label cross-entropy, $\nabla_z\ell(z,y)=q-e_y$, where $e_y$ is the one-hot vector for class $y$. Taking expectation over $y\sim p$ gives $q-p$. The covariance follows because the deterministic $q$ term cancels and $\Cov(e_y)=\diag(p)-pp^\top$.
\end{proof}

\noindent\textbf{Takeaway.} SLS and soft-label training optimize the same objective in expectation. The difference is the additional gradient variance introduced by sampling.

The per-step gradient is $q-e_y$ under SLS and $q-p$ under soft-label training, and
\begin{equation}\label{eq:sqnorm}
\E_{y\sim p}\lVert q-e_y\rVert^2
= \lVert q-p\rVert^2 + 1 - \lVert p\rVert_2^2.
\end{equation}
This additional variance vanishes for one-hot targets and grows with annotator disagreement.

\paragraph{Controls.}
Deterministic control uses the same cycling procedure as multipass with a different fixed shuffle seed. It is not an independent method, but a traversal-order ablation of multipass. Shuffled SLS applies a single permutation to the annotator distributions across examples before training. This preserves marginal disagreement statistics while breaking the pairing between each input and its own annotator distribution.

\section{Experimental Setup}\label{sec:setup}

We evaluate on CIFAR-10H \citep{peterson2019human}, which provides human label distributions over 10{,}000 CIFAR-10 test images. We use only these images with an 80/20 stratified split per seed; no CIFAR-10 training images are used. Results are reported across 10 seeds in the main comparison unless otherwise stated. The completed annotator-count sweep and the hard-delivery ablation use 5 paired seeds.

Our architecture is a CIFAR-adapted ResNet-18 initialized from scratch: the ImageNet-style $7\times7$ stem is replaced by a $3\times3$ convolution with stride 1, and the initial max-pooling layer is removed. Training uses SGD with learning rate $0.1$, momentum $0.9$, weight decay $5\times 10^{-4}$, cosine annealing for 200 epochs, and batch size 128. Training uses random crop and horizontal flip; evaluation uses deterministic normalization-only preprocessing.

Our central comparison is between soft-label training and hard-label delivery. The hard-label family consists of multipass and SLS, with deterministic control and shuffled SLS used as controls in the hard-delivery ablation. For broader context, we also report majority vote, label smoothing, and mixup as standard baselines. Following a proper-scoring-rule view of posterior evaluation, we treat \textbf{soft NLL} as the primary metric and use \textbf{KL-to-annotator} and \textbf{soft Brier} as supporting proper scoring rules \citep{gneiting2007strictly,ferrer2024evaluating}. We report \textbf{hard accuracy} on \emph{all} evaluation examples (\texttt{hard\_acc\_all}), \textbf{equal-mass ECE}, and \textbf{entropy correlation}, defined as the Spearman correlation between predicted-distribution entropy and annotator-distribution entropy across examples, as secondary summaries.

\paragraph{Hessian protocol.}
All Hessian quantities are computed on the best-validation checkpoint (lowest soft NLL) in \texttt{eval} mode with batch-normalization statistics frozen. The top eigenvalue $\lambda_{\max}$ is estimated by power iteration with Hessian-vector products, and the trace is estimated by Hutchinson's method. Both use soft cross-entropy against annotator distributions on the evaluation split. We report both the full evaluation split and a high-disagreement slice.

\paragraph{Ablation construction.}
CIFAR-10H provides integer vote counts per example, which allow us to instantiate multipass directly as defined in Sec.~\ref{sec:method}. Deterministic control uses the same construction with a different fixed shuffle seed. For the annotator-count sweep ($K\in\{5,10,25,50\}$), subsampled counts are drawn via $\texttt{np.random.multinomial}(K,p_i)$ per example, producing integer count vectors that sum exactly to $K$.

\section{Results}
We organize the results around the regime split introduced above: sparse-annotation settings first, full-distribution parity second, and basin geometry as supporting evidence.

\subsection{Sparse-annotation regime}
When only a few annotations per example are available, hard-label delivery is better than soft-label training on soft NLL. Table~\ref{tab:sweep_softnll} shows the same pattern across all three hard-delivery methods and all four annotator counts: every hard-delivery method improves on soft labels at every $K$. In 9 of the 12 method-by-$K$ comparisons, all five paired seeds agree on direction, yielding the minimum raw one-sided Wilcoxon $p=0.03125$. With 5 paired seeds the minimum attainable raw $p$-value is $0.03125$, so Holm correction over the full 12-cell family cannot reject any individual cell even when every seed agrees on direction. We therefore treat the uncorrected $p$-values and the 9/12 perfect-agreement count as the primary evidence, and include Holm as a conservative reference.

\begin{table}[!t]
\centering
\footnotesize
\caption{Soft-NLL results for the annotator-count sweep (5 paired seeds per setting). Entries for hard-delivery methods report mean soft NLL with the raw one-sided paired Wilcoxon $p$ against the corresponding Soft labels row in parentheses. Bold indicates the best value in each column.}
\label{tab:sweep_softnll}
\resizebox{\columnwidth}{!}{%
\begin{tabular}{lcccc}
\toprule
Method & $K=5$ & $K=10$ & $K=25$ & $K=50$ \\
\midrule
Soft labels & 0.5860 & 0.5785 & 0.5388 & 0.5628 \\
SLS & 0.5599 (0.031) & 0.5485 (0.031) & 0.5169 (0.063) & 0.5291 (0.094) \\
Multipass & 0.5649 (0.063) & \textbf{0.5371} (0.031) & 0.5117 (0.031) & 0.5241 (0.031) \\
Deterministic control & \textbf{0.5555} (0.031) & 0.5388 (0.031) & \textbf{0.5077} (0.031) & \textbf{0.5231} (0.031) \\
\bottomrule
\end{tabular}
}
\end{table}

The effect is strongest at $K=5$ and $K=10$, but it persists at $K=25$ and $K=50$, especially for multipass and deterministic control. Hard accuracy remains close across methods, while entropy correlation is modestly higher for hard delivery as $K$ grows. Figure~\ref{fig:annot_sweep} shows the same pattern across the sweep.

\begin{figure*}[!t]
\centering
\includegraphics[width=0.92\textwidth]{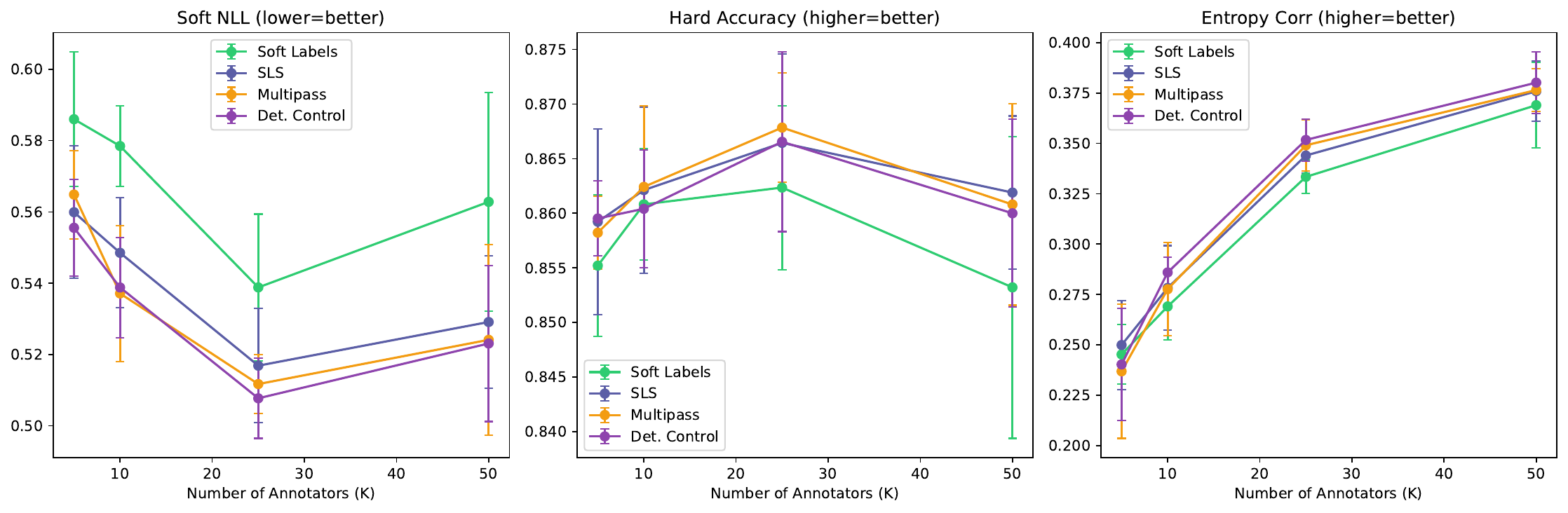}
\caption{Annotator-count sweep on CIFAR-10H. The soft-NLL panel shows the main sparse-regime result: hard-delivery methods are better than soft-label training in all 12 method $\times$ annotator-count cells.}
\label{fig:annot_sweep}
\end{figure*}

Figure~\ref{fig:sparse_target_error} sharpens that result. Hard-delivery improvements are small when the sparse empirical target is already close to the full annotator distribution and grow as that gap increases. Across methods and $K$, the per-seed Spearman correlation between JS distance and hard-delivery improvement is modest but consistently positive, roughly 0.05--0.16, supporting a target-estimation interpretation of the sweep.

\begin{figure*}[!t]
\centering
\includegraphics[width=\textwidth]{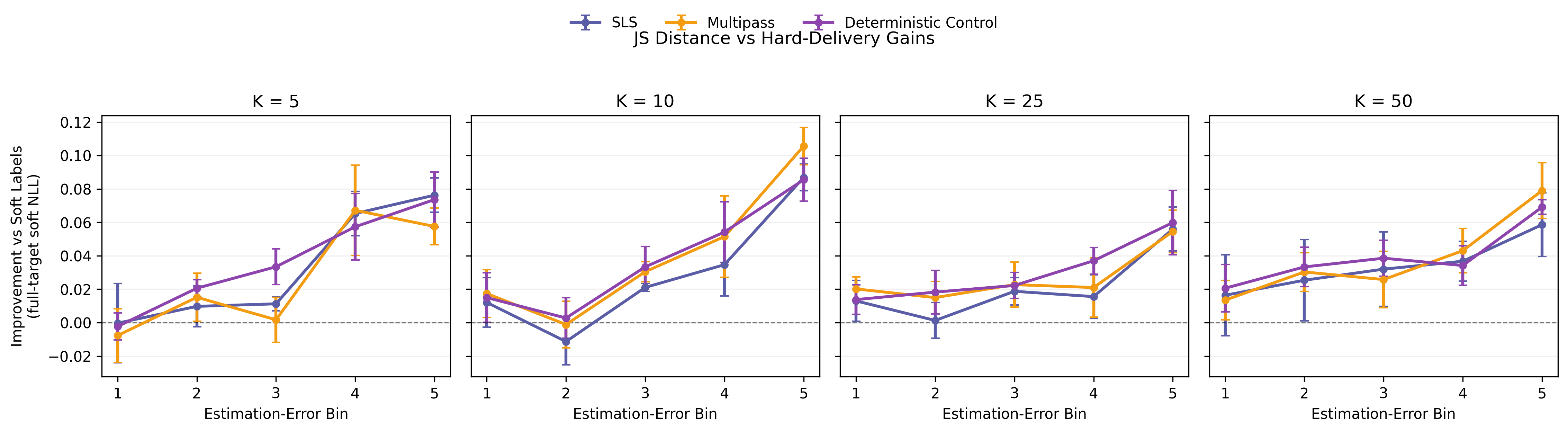}
\caption{Sparse-target diagnostic. Improvements over soft labels remain near zero in the lowest-error bins and grow in higher-error bins.}
\label{fig:sparse_target_error}
\end{figure*}

A resampling-frequency probe is consistent with the same interpretation: holding sampled labels fixed longer monotonically worsens soft NLL, from 0.5027 under every-epoch resampling to 0.5458, 0.6054, and 0.6689 when labels are held fixed for 5, 10, and 50 epochs. Every stale-target condition is worse in all five paired seeds than every-epoch resampling.

\subsection{Full-distribution regime}\label{sec:main}
With full annotator distributions, hard-label delivery and soft-label training reach the same endpoint regime. Table~\ref{tab:main} shows the 10-seed main comparison. SLS does not differ significantly from soft-label training on soft NLL ($p=0.6953$), hard accuracy ($p=0.6250$), equal-mass ECE ($p=0.9219$), or entropy correlation ($p=0.3750$). Both substantially outperform majority vote and label smoothing on posterior-quality metrics, while mixup attains the highest hard accuracy but markedly worse calibration diagnostics.

\begin{table}[!t]
\centering
\small
\caption{Main CIFAR-10H comparison across 10 seeds. Entries are mean $\pm$ across-seed SD. Hard Acc.\ denotes \texttt{hard\_acc\_all}.}
\label{tab:main}
\resizebox{\columnwidth}{!}{%
\begin{tabular}{lcccc}
\toprule
Method & Soft NLL$\downarrow$ & Hard Acc.$\uparrow$ & ECE$_{\mathrm{eqmass}}\downarrow$ & EntCorr$\uparrow$ \\
\midrule
Label smoothing  & 0.6263 $\pm$ 0.0204 & 0.8590 $\pm$ 0.0088 & 0.0598 $\pm$ 0.0059 & 0.2117 $\pm$ 0.0218 \\
Majority vote    & 0.7284 $\pm$ 0.0401 & 0.8570 $\pm$ 0.0107 & 0.0704 $\pm$ 0.0058 & 0.2902 $\pm$ 0.0219 \\
Mixup            & 0.5526 $\pm$ 0.0129 & \textbf{0.8824} $\pm$ 0.0046 & 0.0977 $\pm$ 0.0063 & 0.2499 $\pm$ 0.0145 \\
SLS              & \textbf{0.5052} $\pm$ 0.0263 & \textbf{0.8695} $\pm$ 0.0105 & 0.0186 $\pm$ 0.0039 & \textbf{0.3946} $\pm$ 0.0192 \\
Soft labels      & 0.5096 $\pm$ 0.0233 & 0.8687 $\pm$ 0.0081 & \textbf{0.0185} $\pm$ 0.0040 & 0.3909 $\pm$ 0.0190 \\
\bottomrule
\end{tabular}
}
\end{table}

The main point in this regime is not that SLS wins, but that once each example is trained against the correct annotator distribution, endpoint posterior-quality metrics depend little on whether that target is delivered softly or through hard labels. The reliability curves support the same conclusion: smooth ECE is 0.019 for both SLS and soft labels, versus 0.082 for majority vote, and the SLS and soft-label curves are visually almost indistinguishable (Appendix~\ref{app:reliability}).

\paragraph{Hard-label family consistency.}
The full-distribution parity result extends across the hard-label family. Table~\ref{tab:ablation} shows that multipass, deterministic control, and SLS are statistically indistinguishable on every reported metric in the 5-seed hard-delivery comparison, while shuffled SLS collapses to near-chance accuracy and extremely poor soft NLL. This is the key family check: the count-preserving instance (multipass), its traversal-order ablation (deterministic control), and the lightweight stochastic alternative (SLS) all match soft labels on the full-distribution endpoints, whereas breaking the example-to-distribution match destroys performance.

\begin{table}[!t]
\centering
\small
\caption{Hard-delivery family comparison. The Soft labels row is copied from the main 10-seed summary for context; the remaining rows come from the 5-seed hard-delivery comparison.}
\label{tab:ablation}
\resizebox{\columnwidth}{!}{%
\begin{tabular}{lcccc}
\toprule
Method & Soft NLL$\downarrow$ & Hard Acc.$\uparrow$ & ECE$_{\mathrm{eqmass}}\downarrow$ & EntCorr$\uparrow$ \\
\midrule
Soft labels (ref.)     & 0.5096 $\pm$ 0.0233 & 0.8687 $\pm$ 0.0081 & \textbf{0.0185} $\pm$ 0.0040 & 0.3909 $\pm$ 0.0190 \\
\midrule
SLS                    & 0.5035 $\pm$ 0.0335 & 0.8710 $\pm$ 0.0129 & 0.0193 $\pm$ 0.0045 & 0.3956 $\pm$ 0.0159 \\
Deterministic control  & \textbf{0.4921} $\pm$ 0.0162 & \textbf{0.8724} $\pm$ 0.0050 & 0.0209 $\pm$ 0.0022 & 0.3963 $\pm$ 0.0160 \\
Multipass              & 0.4942 $\pm$ 0.0123 & 0.8714 $\pm$ 0.0063 & \textbf{0.0191} $\pm$ 0.0044 & \textbf{0.4000} $\pm$ 0.0287 \\
Shuffled SLS           & 2.2973 $\pm$ 0.0024 & 0.1199 $\pm$ 0.0128 & 0.0293 $\pm$ 0.0099 & $-0.0058$ $\pm$ 0.0556 \\
\bottomrule
\end{tabular}
}
\end{table}

Against SLS, multipass has $p=0.4375$ on soft NLL and $p\ge 0.625$ on the other metrics; deterministic control shows the same pattern. Shuffled SLS, by contrast, drops to 12\% hard accuracy and attains the minimum attainable two-sided $p=0.0625$ on several columns with only five paired seeds.

\subsection{Geometry and supporting diagnostics}\label{sec:geometry}
The proposition predicts that SLS injects label-sampling variance proportional to annotator disagreement. We verify that prediction empirically in Appendix~\ref{app:grad_variance}: across seeds, the Spearman correlation between annotator entropy and last-layer gradient variance averages 0.939 and ranges from 0.919 to 0.953. The extra optimization noise introduced by SLS is therefore highly structured rather than incidental, and is concentrated on the examples where annotators disagree most.

Table~\ref{tab:hessian} shows the main geometry result. In the 10-seed main comparison, SLS has $\lambda_{\max}=104.9\pm 8.8$ versus $242.2\pm 36.2$ for soft labels, and trace $1633.9\pm 50.6$ versus $4946.0\pm 943.2$. The 5-seed hard-delivery comparison matches that pattern: multipass, deterministic control, and SLS occupy essentially the same flatness regime, while shuffled SLS is formally flattest but achieves near-chance accuracy. Because deterministic control has zero sampling stochasticity, this ordering argues against stochasticity alone as the primary explanation.

\begin{table}[!t]
\centering
\small
\caption{Hessian geometry across methods on CIFAR-10H. ``Full'' uses the full evaluation split and ``high'' restricts to the high-disagreement slice.}
\label{tab:hessian}
\resizebox{\columnwidth}{!}{%
\begin{tabular}{lcccc}
\toprule
Method & $\lambda_{\max}$ (full) & Trace (full) & $\lambda_{\max}$ (high) & Trace (high) \\
\midrule
SLS (main)     & $104.9 \pm 8.8$   & $1633.9 \pm 50.6$  & $314.0 \pm 57.3$  & $3496.8 \pm 360.3$ \\
Soft labels    & $242.2 \pm 36.2$  & $4946.0 \pm 943.2$ & $780.6 \pm 103.3$ & $10779.6 \pm 1811.9$ \\
\midrule
SLS (ablation) & $107.8 \pm 9.9$   & $1642.9 \pm 25.7$  & $310.3 \pm 22.0$  & $3480.4 \pm 255.2$ \\
Deterministic control & $101.6 \pm 8.1$   & $1581.6 \pm 61.2$  & $314.6 \pm 30.4$  & $3481.8 \pm 250.7$ \\
Multipass      & $103.8 \pm 4.3$   & $1571.4 \pm 72.4$  & $302.2 \pm 29.6$  & $3438.4 \pm 339.5$ \\
Shuffled SLS   & $18.4 \pm 13.3$   & $33.9 \pm 26.0$    & $20.5 \pm 18.1$   & $38.8 \pm 44.9$ \\
\bottomrule
\end{tabular}
}
\end{table}

Additional diagnostics support the same picture. The mean loss barrier between SLS and soft-label checkpoints is 2.05, which is well above zero and therefore consistent with the two solutions occupying distinct basins despite near-identical endpoint loss. SLS also yields more reproducible representations: within-method CKA is 0.920 for SLS versus 0.887 for soft labels, and Grad-CAM cross-seed stability is 0.901 versus 0.804 overall, and 0.861 versus 0.756 on the high-entropy slice.

OOD detection provides descriptive support for the same pattern. On SVHN, hard-label delivery exceeds soft-label training on five of six detectors in the family comparison. On CIFAR-100, SLS improves AUROC over soft labels on every reported score type, with paired $p=0.0186$ on Energy and ODIN. We report the full tables in Appendix~\ref{app:ood} and treat the family-level comparisons as descriptive rather than paired hypothesis tests, since the soft-label column comes from the 10-seed main comparison whereas the hard-label family entries come from the 5-seed ablation.

\section{Discussion and Conclusion}
The results point to a clear regime split. When only a few annotations per example are available, hard-label delivery is directionally better than soft-label training on soft NLL, with larger improvements where the sparse empirical target is farther from the full annotator distribution. When full annotator distributions are available, that difference disappears at the endpoint: multipass, deterministic control, SLS, and soft-label training all reach the same performance regime, while shuffled SLS fails once the example-to-distribution match is broken.

Taken together, these results suggest that preserving the example-to-distribution match is the first-order ingredient in the full-distribution regime, while delivery format matters more when the empirical target is sparse or stale. In that regime, the choice of delivery format begins to matter not only for optimization path and basin geometry, but also for endpoint soft NLL.

This also helps place the paper in the epistemic-intelligence setting. We interpret annotator disagreement as useful here not because it behaves like arbitrary label corruption, but because it carries information about uncertainty in the target itself. The full-distribution results show that this uncertainty is best preserved by keeping each example matched to its own annotator distribution. The sparse-regime results then show that when this uncertainty is only partially observed, delivery format begins to matter. This is the sense in which the paper treats disagreement as epistemic signal rather than as noise to be collapsed away.

The geometry results support the same picture. Hard-label delivery converges to flatter basins than soft-label training, and the descriptive OOD comparisons point in the same direction. Deterministic control is useful here because it places those effects on hard-label delivery itself rather than on sampling stochasticity alone.

From a practical standpoint, multipass is the recommended default when raw vote counts are available: it is count-preserving, deterministic, and reproducible. SLS is the lightweight alternative when only per-example probabilities are available, for example when distributions come from a pretrained teacher or a smoothed estimator. Deterministic control is an internal validity check rather than a method users are meant to deploy.

\paragraph{Limitations.}
This study is centered on a single dataset, CIFAR-10H, and a single primary architecture, a CIFAR-adapted ResNet-18, so transfer to NLP disagreement benchmarks such as ChaosNLI or LeWiDi, or to larger vision models, remains to be tested. The main 10-seed SLS-versus-soft-label comparison and the 5-seed hard-delivery comparison use different seed budgets, and the family-level OOD tables mix these two sources; we therefore treat those OOD comparisons as descriptive support rather than paired hypothesis tests. In the sparse annotator-count sweep, all 12 hard-delivery cells are directionally better than soft labels on soft NLL, but none survives Holm correction over the 12-cell family. The sparse-annotation conditions are also simulated by subsampling CIFAR-10H's dense annotation pool, which may not reproduce the annotator-selection and task-specific disagreement structure of genuinely sparse datasets. Finally, the geometry analyses are observational: Proposition~\ref{prop:main} gives a per-step variance decomposition, but the link from delivery format to basin geometry is characterized rather than derived.

\bibliographystyle{plainnat}
\bibliography{references}

@article{dawid1979maximum,
  author  = {Dawid, A. Philip and Skene, Allan M.},
  title   = {Maximum Likelihood Estimation of Observer Error-Rates Using the {EM} Algorithm},
  journal = {Journal of the Royal Statistical Society: Series C (Applied Statistics)},
  volume  = {28},
  number  = {1},
  pages   = {20--28},
  year    = {1979},
  doi     = {10.2307/2346806}
}

@article{raykar2010learning,
  author  = {Raykar, Vikas C. and Yu, Shipeng and Zhao, Linda H. and Valadez, Gerardo Hermosillo and Florin, Charles and Bogoni, Luca and Moy, Linda},
  title   = {Learning From Crowds},
  journal = {Journal of Machine Learning Research},
  volume  = {11},
  pages   = {1297--1322},
  year    = {2010},
  url     = {https://jmlr.org/papers/v11/raykar10a.html}
}

@article{aroyo2015truth,
  author  = {Aroyo, Lora and Welty, Chris},
  title   = {Truth Is a Lie: Crowd Truth and the Seven Myths of Human Annotation},
  journal = {AI Magazine},
  volume  = {36},
  number  = {1},
  pages   = {15--24},
  year    = {2015},
  doi     = {10.1609/aimag.v36i1.2564}
}

@inproceedings{peterson2019human,
  author    = {Peterson, Joshua C. and Battleday, Ruairidh M. and Griffiths, Thomas L. and Russakovsky, Olga},
  title     = {Human Uncertainty Makes Classification More Robust},
  booktitle = {Proceedings of the IEEE/CVF International Conference on Computer Vision (ICCV)},
  pages     = {9617--9626},
  year      = {2019},
  doi       = {10.1109/ICCV.2019.00971}
}

@article{battleday2020capturing,
  author  = {Battleday, Ruairidh M. and Peterson, Joshua C. and Griffiths, Thomas L.},
  title   = {Capturing Human Categorization of Natural Images by Combining Deep Networks and Cognitive Models},
  journal = {Nature Communications},
  volume  = {11},
  pages   = {5418},
  year    = {2020},
  doi     = {10.1038/s41467-020-18946-z}
}

@article{uma2021learning,
  author  = {Uma, Alexandra N. and Fornaciari, Tommaso and Hovy, Dirk and Paun, Silviu and Plank, Barbara and Poesio, Massimo},
  title   = {Learning From Disagreement: A Survey},
  journal = {Journal of Artificial Intelligence Research},
  volume  = {72},
  pages   = {1385--1470},
  year    = {2021},
  doi     = {10.1613/jair.1.12752}
}

@inproceedings{plank2022problem,
  author    = {Plank, Barbara},
  title     = {The ``Problem'' of Human Label Variation: On Ground Truth in Data, Modeling and Evaluation},
  booktitle = {Proceedings of the 2022 Conference on Empirical Methods in Natural Language Processing (EMNLP)},
  pages     = {10671--10682},
  year      = {2022},
  doi       = {10.18653/v1/2022.emnlp-main.731}
}

@article{mostafazadehdavani2022dealing,
  author  = {Mostafazadeh Davani, Aida and D{\'\i}az, Mark and Prabhakaran, Vinodkumar},
  title   = {Dealing with Disagreements: Looking Beyond the Majority Vote in Subjective Annotations},
  journal = {Transactions of the Association for Computational Linguistics},
  volume  = {10},
  pages   = {92--110},
  year    = {2022},
  doi     = {10.1162/tacl_a_00449}
}

@article{frenda2025perspectivist,
  author  = {Frenda, Simona and Abercrombie, Gavin and Basile, Valerio and Pedrani, Alessandro and Panizzon, Raffaella and Cignarella, Alessandra Teresa and Marco, Cristina and Bernardi, Davide},
  title   = {Perspectivist Approaches to Natural Language Processing: A Survey},
  journal = {Language Resources and Evaluation},
  volume  = {59},
  number  = {2},
  pages   = {1719--1746},
  year    = {2025},
  doi     = {10.1007/s10579-024-09766-4}
}

@article{xu2026beyond,
  author  = {Xu, Yinuo and Jurgens, David},
  title   = {Beyond Consensus: Perspectivist Modeling and Evaluation of Annotator Disagreement in {NLP}},
  journal = {arXiv preprint arXiv:2601.09065},
  year    = {2026},
  url     = {https://arxiv.org/abs/2601.09065}
}

@inproceedings{nie2020what,
  author    = {Nie, Yixin and Zhou, Xiang and Bansal, Mohit},
  title     = {What Can We Learn from Collective Human Opinions on Natural Language Inference Data?},
  booktitle = {Proceedings of the 2020 Conference on Empirical Methods in Natural Language Processing (EMNLP)},
  pages     = {9131--9143},
  year      = {2020},
  doi       = {10.18653/v1/2020.emnlp-main.734}
}

@inproceedings{leonardelli2023semeval,
  author    = {Leonardelli, Elisa and Abercrombie, Gavin and Almanea, Dina and Basile, Valerio and Fornaciari, Tommaso and Plank, Barbara and Rieser, Verena and Uma, Alexandra and Poesio, Massimo},
  title     = {{SemEval-2023} {Task} 11: Learning with Disagreements ({LeWiDi})},
  booktitle = {Proceedings of the 17th International Workshop on Semantic Evaluation (SemEval-2023)},
  pages     = {2304--2318},
  year      = {2023},
  doi       = {10.18653/v1/2023.semeval-1.314}
}

@inproceedings{sheng2008get,
  author    = {Sheng, Victor S. and Provost, Foster and Ipeirotis, Panagiotis G.},
  title     = {Get Another Label? {I}mproving Data Quality and Data Mining Using Multiple, Noisy Labelers},
  booktitle = {Proceedings of the 14th ACM SIGKDD International Conference on Knowledge Discovery and Data Mining (KDD)},
  pages     = {614--622},
  year      = {2008},
  doi       = {10.1145/1401890.1401965}
}

@article{hinton2015distilling,
  author  = {Hinton, Geoffrey and Vinyals, Oriol and Dean, Jeff},
  title   = {Distilling the Knowledge in a Neural Network},
  journal = {arXiv preprint arXiv:1503.02531},
  year    = {2015},
  url     = {https://arxiv.org/abs/1503.02531}
}

@inproceedings{szegedy2016rethinking,
  author    = {Szegedy, Christian and Vanhoucke, Vincent and Ioffe, Sergey and Shlens, Jonathon and Wojna, Zbigniew},
  title     = {Rethinking the Inception Architecture for Computer Vision},
  booktitle = {Proceedings of the IEEE Conference on Computer Vision and Pattern Recognition (CVPR)},
  pages     = {2818--2826},
  year      = {2016},
  doi       = {10.1109/CVPR.2016.308}
}

@inproceedings{muller2019when,
  author    = {M{\"u}ller, Rafael and Kornblith, Simon and Hinton, Geoffrey},
  title     = {When Does Label Smoothing Help?},
  booktitle = {Advances in Neural Information Processing Systems (NeurIPS)},
  volume    = {32},
  year      = {2019},
  url       = {https://proceedings.neurips.cc/paper/2019/hash/f1748d6b0fd9d439f71450117eba2725-Abstract.html}
}

@inproceedings{zhang2018mixup,
  author    = {Zhang, Hongyi and Ciss{\'e}, Moustapha and Dauphin, Yann N. and Lopez-Paz, David},
  title     = {mixup: Beyond Empirical Risk Minimization},
  booktitle = {International Conference on Learning Representations (ICLR)},
  year      = {2018},
  url       = {https://openreview.net/forum?id=r1Ddp1-Rb}
}

@inproceedings{xie2016disturblabel,
  author    = {Xie, Lingxi and Wang, Jingdong and Wei, Zhen and Wang, Meng and Tian, Qi},
  title     = {{DisturbLabel}: Regularizing {CNN} on the Loss Layer},
  booktitle = {Proceedings of the IEEE Conference on Computer Vision and Pattern Recognition (CVPR)},
  pages     = {4753--4762},
  year      = {2016},
  doi       = {10.1109/CVPR.2016.514}
}

@inproceedings{natarajan2013learning,
  author    = {Natarajan, Nagarajan and Dhillon, Inderjit S. and Ravikumar, Pradeep and Tewari, Ambuj},
  title     = {Learning with Noisy Labels},
  booktitle = {Advances in Neural Information Processing Systems (NeurIPS)},
  volume    = {26},
  year      = {2013},
  url       = {https://proceedings.neurips.cc/paper/2013/hash/3871bd64012152bfb53fdf04b401193f-Abstract.html}
}

@article{frenay2014classification,
  author  = {Fr{\'e}nay, Beno{\^\i}t and Verleysen, Michel},
  title   = {Classification in the Presence of Label Noise: A Survey},
  journal = {IEEE Transactions on Neural Networks and Learning Systems},
  volume  = {25},
  number  = {5},
  pages   = {845--869},
  year    = {2014},
  doi     = {10.1109/TNNLS.2013.2292894}
}

@article{song2023learning,
  author  = {Song, Hwanjun and Kim, Minseok and Park, Dongmin and Shin, Yooju and Lee, Jae-Gil},
  title   = {Learning from Noisy Labels with Deep Neural Networks: A Survey},
  journal = {IEEE Transactions on Neural Networks and Learning Systems},
  volume  = {34},
  number  = {11},
  pages   = {8135--8153},
  year    = {2023},
  doi     = {10.1109/TNNLS.2022.3152527}
}

@article{gneiting2007strictly,
  author  = {Gneiting, Tilmann and Raftery, Adrian E.},
  title   = {Strictly Proper Scoring Rules, Prediction, and Estimation},
  journal = {Journal of the American Statistical Association},
  volume  = {102},
  number  = {477},
  pages   = {359--378},
  year    = {2007},
  doi     = {10.1198/016214506000001437}
}

@article{ferrer2024evaluating,
  author  = {Ferrer, Luciana and Ramos, Daniel},
  title   = {Evaluating Posterior Probabilities: Decision Theory, Proper Scoring Rules, and Calibration},
  journal = {arXiv preprint arXiv:2408.02841},
  year    = {2024},
  url     = {https://arxiv.org/abs/2408.02841}
}

@article{hochreiter1997flat,
  author  = {Hochreiter, Sepp and Schmidhuber, J{\"u}rgen},
  title   = {Flat Minima},
  journal = {Neural Computation},
  volume  = {9},
  number  = {1},
  pages   = {1--42},
  year    = {1997},
  doi     = {10.1162/neco.1997.9.1.1}
}

@inproceedings{keskar2017largebatch,
  author    = {Keskar, Nitish Shirish and Mudigere, Dheevatsa and Nocedal, Jorge and Smelyanskiy, Mikhail and Tang, Ping Tak Peter},
  title     = {On Large-Batch Training for Deep Learning: Generalization Gap and Sharp Minima},
  booktitle = {International Conference on Learning Representations (ICLR)},
  year      = {2017},
  url       = {https://openreview.net/forum?id=H1oyRlYgg}
}

@inproceedings{smith2021origin,
  author    = {Smith, Samuel L. and Dherin, Benoit and Barrett, David G. T. and De, Soham},
  title     = {On the Origin of Implicit Regularization in Stochastic Gradient Descent},
  booktitle = {International Conference on Learning Representations (ICLR)},
  year      = {2021},
  url       = {https://openreview.net/forum?id=rq_Qr0c1Hyo}
}

@inproceedings{damian2021label,
  author    = {Damian, Alex and Ma, Tengyu and Lee, Jason D.},
  title     = {Label Noise {SGD} Provably Prefers Flat Global Minimizers},
  booktitle = {Advances in Neural Information Processing Systems (NeurIPS)},
  volume    = {34},
  year      = {2021},
  url       = {https://proceedings.neurips.cc/paper/2021/hash/e6af401c28c1790eaef7d55c92ab6ab6-Abstract.html}
}

@inproceedings{haochen2021shape,
  author    = {HaoChen, Jeff Z. and Wei, Colin and Lee, Jason D. and Ma, Tengyu},
  title     = {Shape Matters: Understanding the Implicit Bias of the Noise Covariance},
  booktitle = {Conference on Learning Theory (COLT)},
  year      = {2021},
  url       = {https://arxiv.org/abs/2006.08680}
}

@inproceedings{wu2022alignment,
  author    = {Wu, Lei and Wang, Mingze and Su, Weijie J.},
  title     = {The Alignment Property of {SGD} Noise and How It Helps Select Flat Minima: A Stability Analysis},
  booktitle = {Advances in Neural Information Processing Systems (NeurIPS)},
  volume    = {35},
  year      = {2022},
  url       = {https://arxiv.org/abs/2207.02628}
}

@inproceedings{rodrigues2018deep,
  author    = {Filipe Rodrigues and Francisco C. Pereira},
  title     = {Deep Learning from Crowds},
  booktitle = {Proceedings of the AAAI Conference on Artificial Intelligence},
  volume    = {32},
  pages     = {1611--1618},
  year      = {2018}
}

\clearpage
\appendix

\section{Supplementary Material}
\label{sec:appendix}

This supplement collects supporting analyses, additional diagnostics, and implementation details.

\subsection{Reproducibility Details}\label{app:repro}

\paragraph{Ablation construction.}
CIFAR-10H provides raw integer vote counts per example. \emph{Multipass} expands each count vector into the corresponding multiset of observed labels, shuffles that list once per example using a fixed seed, and cycles through it across epochs via $\texttt{epoch} \bmod \texttt{len(sequence)}$. \emph{Deterministic control} is identical except for a different fixed shuffle seed (offset by 1000), and should therefore be read as an ablation of multipass rather than as an independent method. \emph{Shuffled SLS} applies a single permutation to the annotator distributions across examples before training, preserving marginal disagreement statistics while destroying the input-to-distribution mapping.

For the annotator-count sweep ($K \in \{5,10,25,50\}$), subsampled counts are drawn via $\texttt{np.random.multinomial}(K, p_i)$ per example. The subsampling seed is coupled to the seed that also determines the train/eval split, so cross-seed comparisons include both split variation and subsampling variation. The different $K$ settings are generated as independent multinomial draws rather than nested subsets.

\paragraph{Hessian protocol.}
All Hessian quantities are computed on the best-validation checkpoint in \texttt{eval} mode with batch-normalization statistics frozen. The top eigenvalue $\lambda_{\max}$ is estimated by power iteration with Hessian-vector products, and the trace by Hutchinson's estimator. The ``full'' columns use the complete evaluation split; the ``high'' columns restrict to the high-disagreement slice.

\subsection{Full-Distribution Diagnostics}

\paragraph{Reliability plots.}
\label{app:reliability}
The reliability curves support the full-distribution parity result reported in the main text.

\begin{figure*}[!t]
\centering
\includegraphics[width=0.92\textwidth]{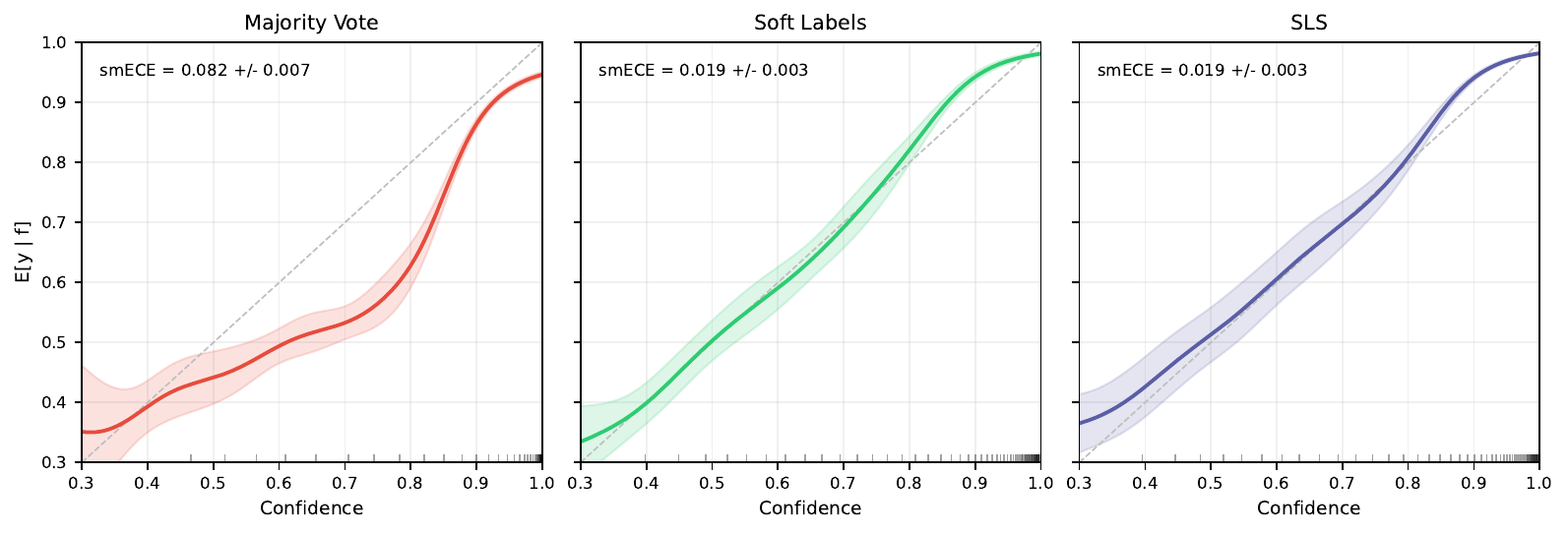}
\caption{Smooth reliability diagrams for majority vote, soft labels, and SLS from the 10-seed main comparison. Majority vote is systematically worse, while SLS and soft-label training closely track the diagonal and remain visually almost indistinguishable.}
\label{fig:smooth_reliability}
\end{figure*}

\paragraph{Additional proper scoring rules.}
Table~\ref{tab:scoring_rules} reports two supporting proper scoring rules. The same pattern holds: SLS and soft labels remain closely matched, while majority vote and label smoothing are substantially worse.

\begin{table}[!t]
\centering
\small
\caption{Additional proper scoring rules on CIFAR-10H. Entries are mean $\pm$ across-seed SD over 10 seeds.}
\label{tab:scoring_rules}
\begin{tabular}{lcc}
\toprule
Method & KL-to-annotator$\downarrow$ & Soft Brier$\downarrow$ \\
\midrule
SLS              & \textbf{0.3515} $\pm$ 0.0260 & 0.1553 $\pm$ 0.0110 \\
Soft labels      & 0.3558 $\pm$ 0.0228 & \textbf{0.1548} $\pm$ 0.0096 \\
Majority vote    & 0.5746 $\pm$ 0.0397 & 0.1943 $\pm$ 0.0152 \\
Label smoothing  & 0.4726 $\pm$ 0.0204 & 0.1786 $\pm$ 0.0094 \\
Mixup            & 0.3988 $\pm$ 0.0123 & 0.1500 $\pm$ 0.0049 \\
\bottomrule
\end{tabular}
\end{table}

\subsection{Sparse-Target Diagnostics}

\paragraph{Annotator-count Hessian sweep.}
Table~\ref{tab:hessian_sweep} shows that the flatness gap persists across the annotator-count sweep. At every tested $K$, SLS remains flatter than soft-label training on both the full split and the high-disagreement slice.

\begin{table}[!t]
\centering
\footnotesize
\caption{Top Hessian eigenvalues in the five-seed annotator-count sweep. SLS remains flatter than soft-label training across all $K$.}
\label{tab:hessian_sweep}
\resizebox{\columnwidth}{!}{%
\begin{tabular}{lcccc}
\toprule
$K$ & SLS (full) & Soft labels (full) & SLS (high) & Soft labels (high) \\
\midrule
5  & \textbf{114.0} $\pm$ 4.5 & 241.4 $\pm$ 36.6 & \textbf{478.7} $\pm$ 76.6 & 1203.1 $\pm$ 158.2 \\
10 & \textbf{114.1} $\pm$ 8.0 & 243.6 $\pm$ 47.9 & \textbf{369.1} $\pm$ 52.2 & 849.7 $\pm$ 296.3 \\
25 & \textbf{97.5} $\pm$ 4.0 & 267.1 $\pm$ 43.7 & \textbf{277.8} $\pm$ 35.8 & 876.1 $\pm$ 235.4 \\
50 & \textbf{98.9} $\pm$ 1.2 & 218.4 $\pm$ 20.0 & \textbf{312.4} $\pm$ 45.8 & 626.4 $\pm$ 38.3 \\
\bottomrule
\end{tabular}
}
\end{table}

\paragraph{Resampling frequency.}
Figure~\ref{fig:resample_freq} shows the resampling-frequency probe. Soft NLL worsens monotonically as sampled labels are held fixed longer.

\begin{figure}[!t]
\centering
\includegraphics[width=\columnwidth]{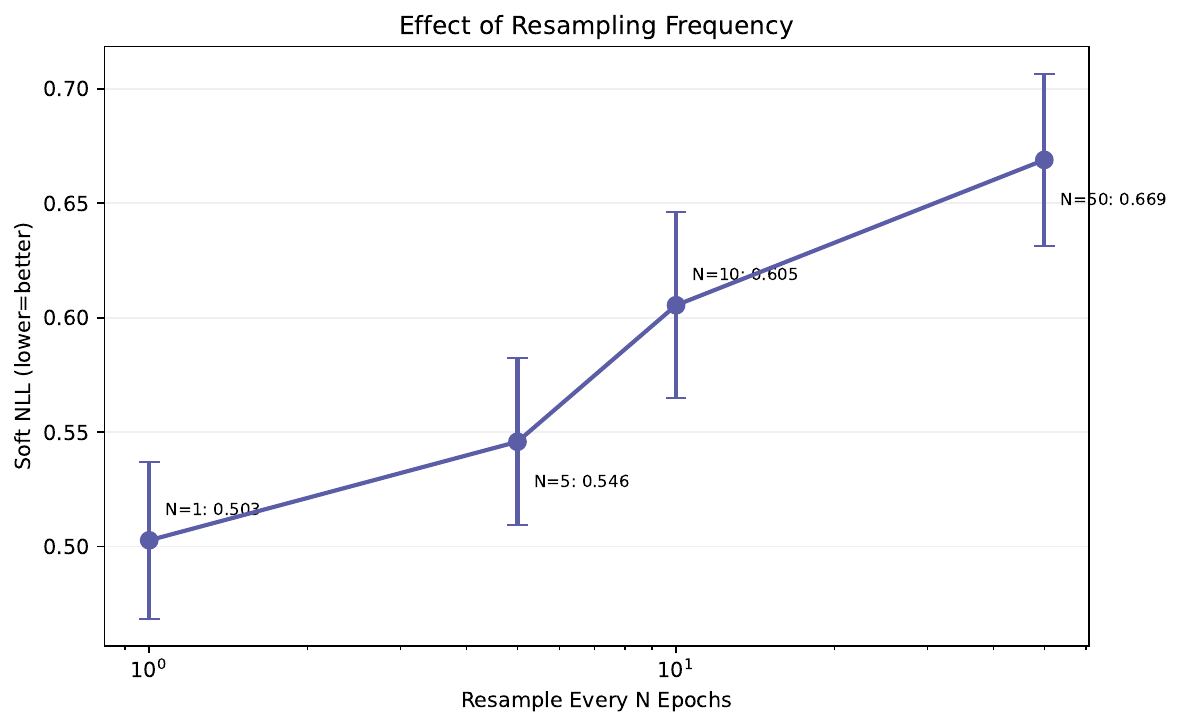}
\caption{Resampling-frequency probe. Mean soft NLL degrades as sampled labels are held fixed longer: every epoch 0.5027, every 5 epochs 0.5458, every 10 epochs 0.6054, and every 50 epochs 0.6689.}
\label{fig:resample_freq}
\end{figure}

\paragraph{Robustness of the sparse-target diagnostic.}
The JS-distance diagnostic in the main text is the cleanest summary, but the same qualitative picture appears under alternative views.

\begin{figure*}[!tp]
\centering
\includegraphics[width=0.92\textwidth]{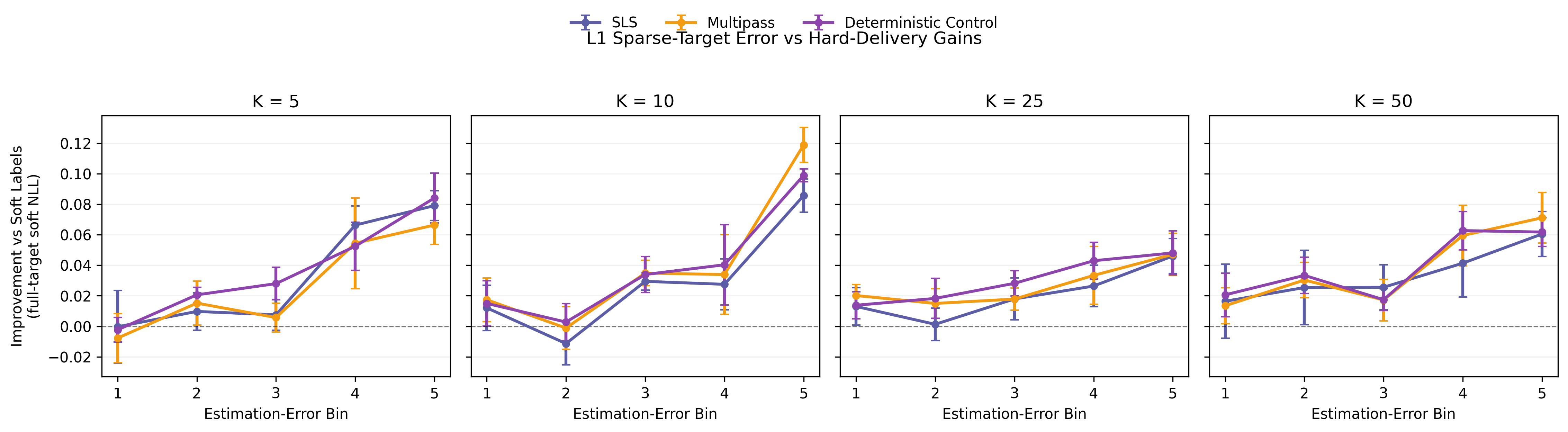}
\caption{Sparse-target robustness check using L1 error instead of JS distance. The qualitative pattern matches the main-text figure: hard-delivery improvements remain near zero in low-error bins and grow in higher-error bins.}
\label{fig:sparse_target_error_l1}
\end{figure*}

\begin{figure*}[!tp]
\centering
\includegraphics[width=0.92\textwidth]{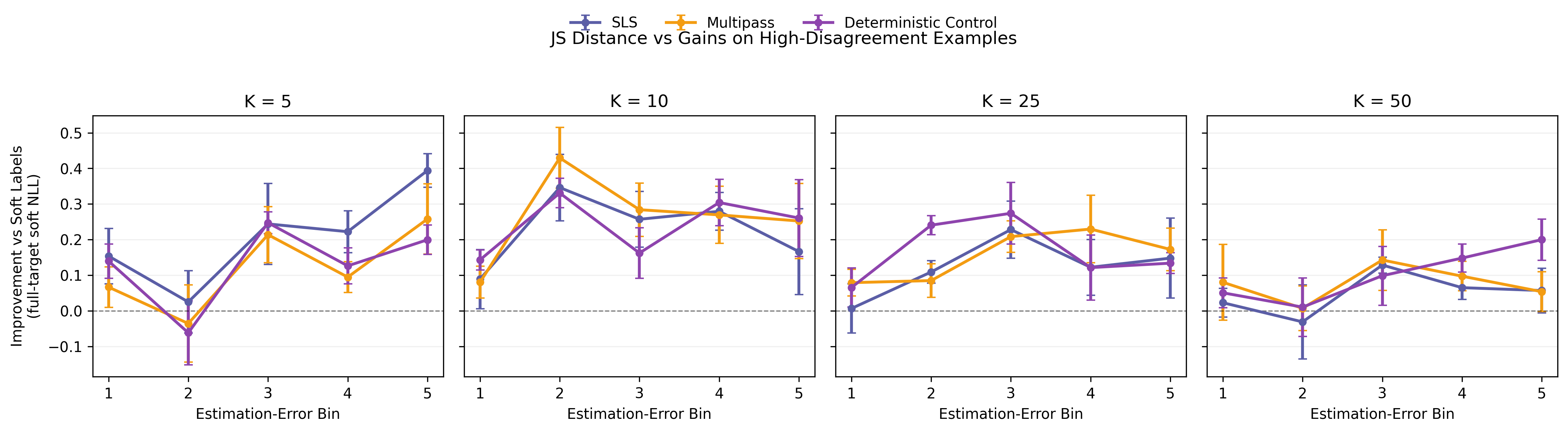}
\caption{Sparse-target robustness check restricted to the high-disagreement slice. The slice is noisier because it is much smaller, but the highest-error bins still show the largest improvements.}
\label{fig:sparse_target_error_high}
\end{figure*}

\subsection{Geometry and Representation Diagnostics}
\label{app:gradvar}

\paragraph{Gradient variance.}
\label{app:grad_variance}
The proposition predicts that SLS injects label-sampling variance proportional to annotator disagreement. Figure~\ref{fig:grad_variance} supports that prediction: the per-seed Spearman correlation between annotator entropy and last-layer gradient variance averages 0.939 and ranges from 0.919 to 0.953.

\begin{figure}[!t]
\centering
\includegraphics[width=\columnwidth]{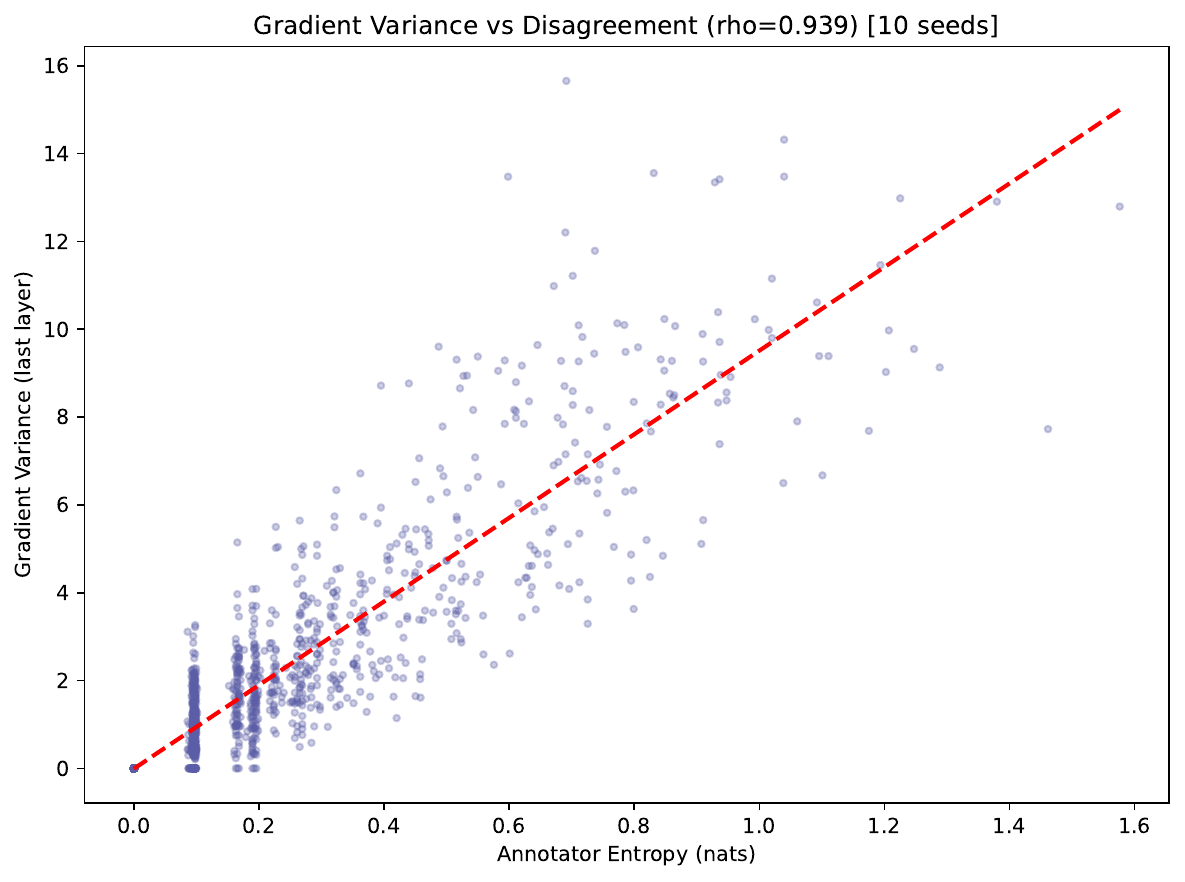}
\caption{Gradient-variance analysis for the main CIFAR-10H setting. Across runs, the Spearman correlation between annotator entropy and last-layer gradient variance averages about 0.939.}
\label{fig:grad_variance}
\end{figure}

\paragraph{Additional supporting diagnostics.}
The following figures provide compact supporting views of the geometry and representation probes reported in the main text.

\begin{figure*}[!tp]
\centering
\includegraphics[width=0.92\textwidth]{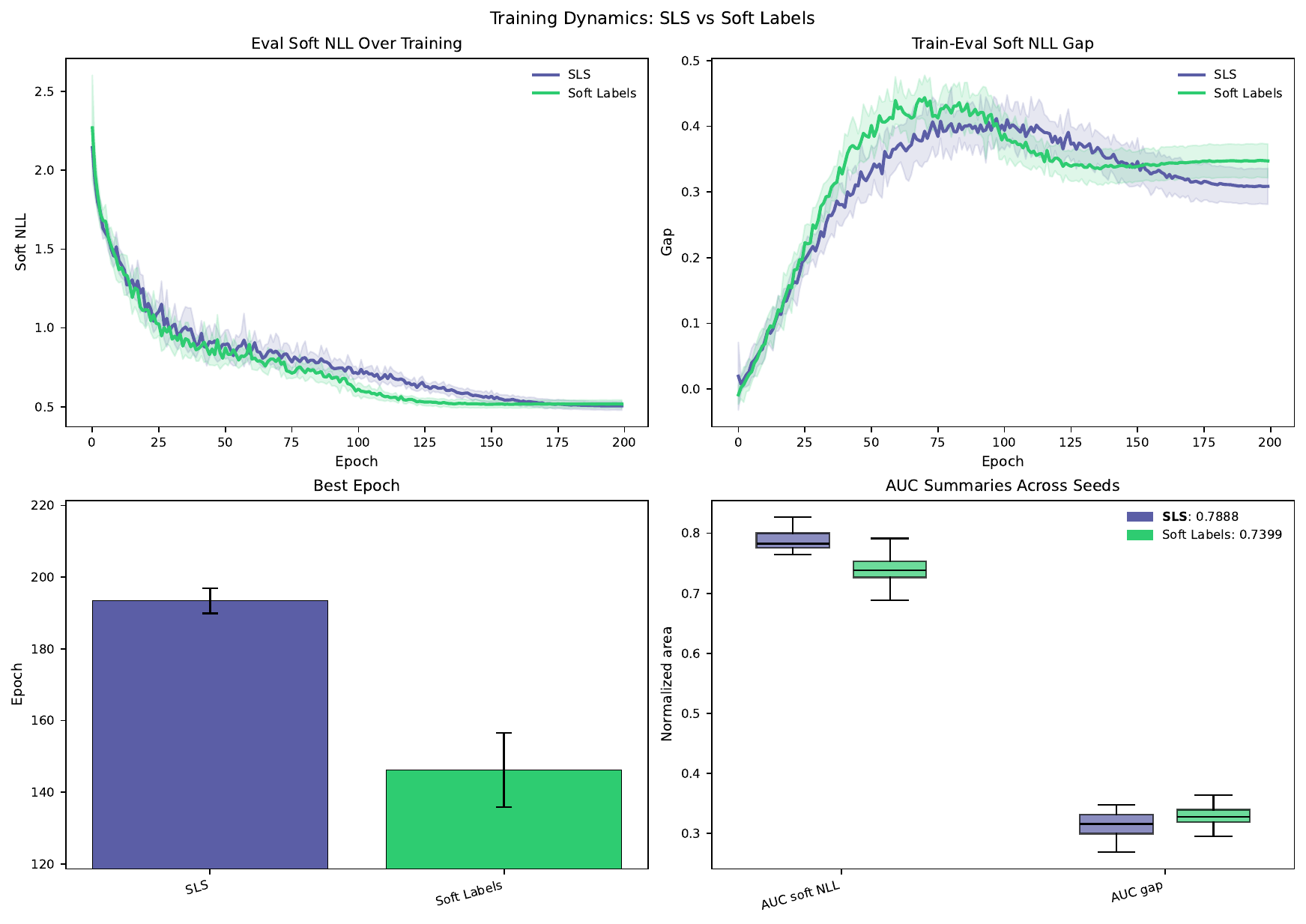}
\caption{Training-dynamics comparison. Averaged across seeds, SLS reaches 95\% of its best eval soft NLL about 35.7 epochs earlier than soft labels and reaches its best epoch about 47.2 epochs earlier, while the final endpoints remain close.}
\label{fig:training_dynamics_appendix}
\end{figure*}

\begin{figure*}[!tp]
\centering
\includegraphics[width=0.92\textwidth]{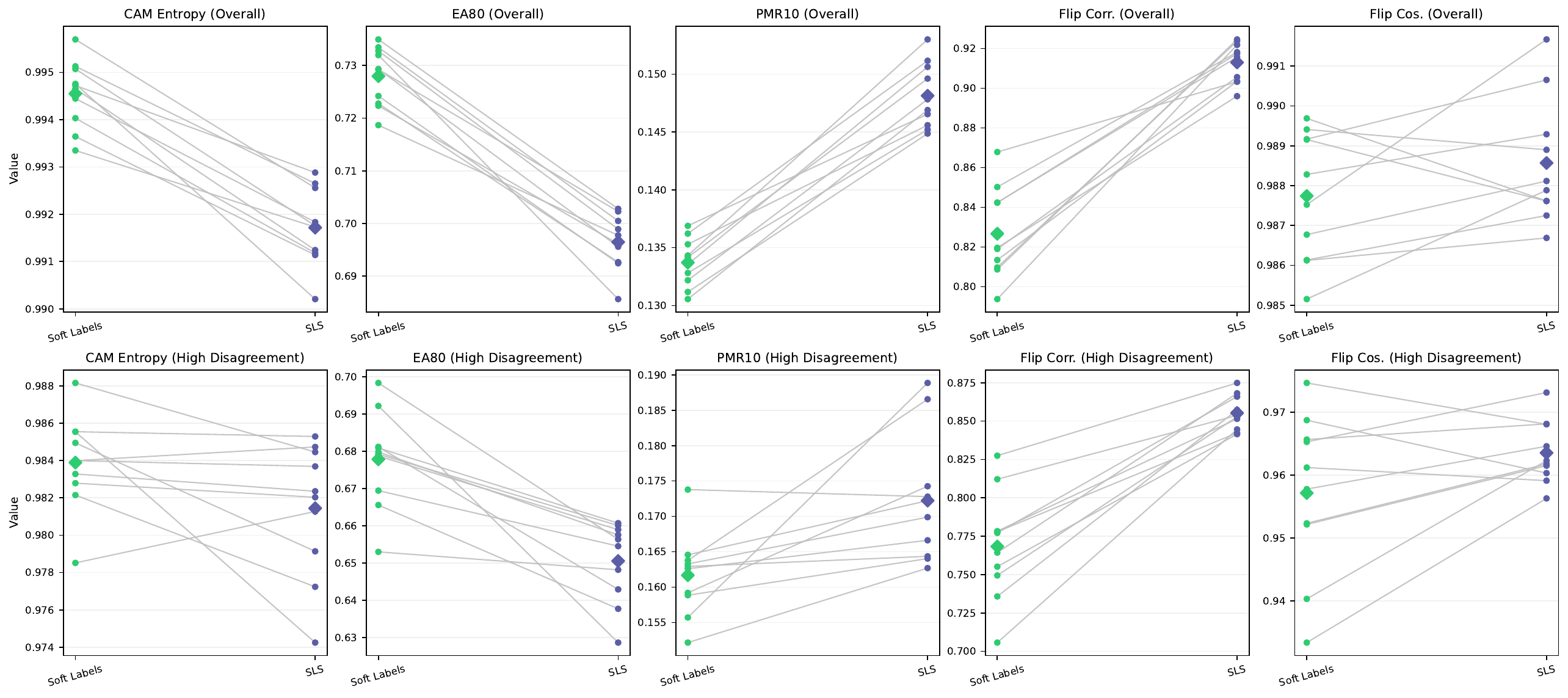}
\caption{Quantitative Grad-CAM stability. The strongest signal is cross-seed stability correlation: 0.901 vs.\ 0.804 overall in favor of SLS, and 0.861 vs.\ 0.756 on the high-entropy slice.}
\label{fig:gradcam_quant_appendix}
\end{figure*}

\begin{figure*}[!tp]
\centering
\includegraphics[width=0.9\textwidth]{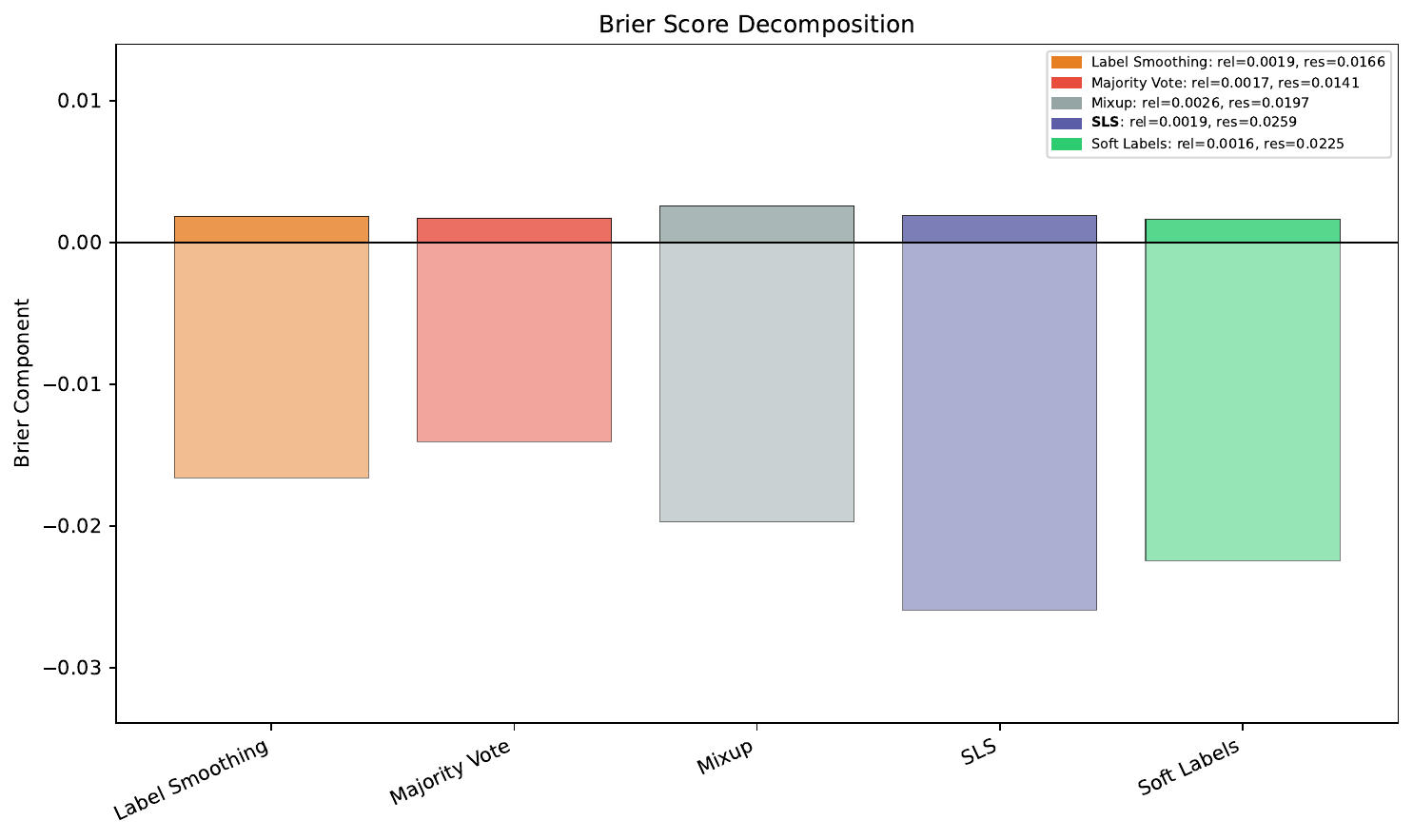}
\caption{Brier decomposition. SLS and soft labels remain close on reliability while SLS shows slightly higher resolution.}
\label{fig:brier_decomp}
\end{figure*}

\begin{figure*}[!tp]
\centering
\includegraphics[width=0.9\textwidth]{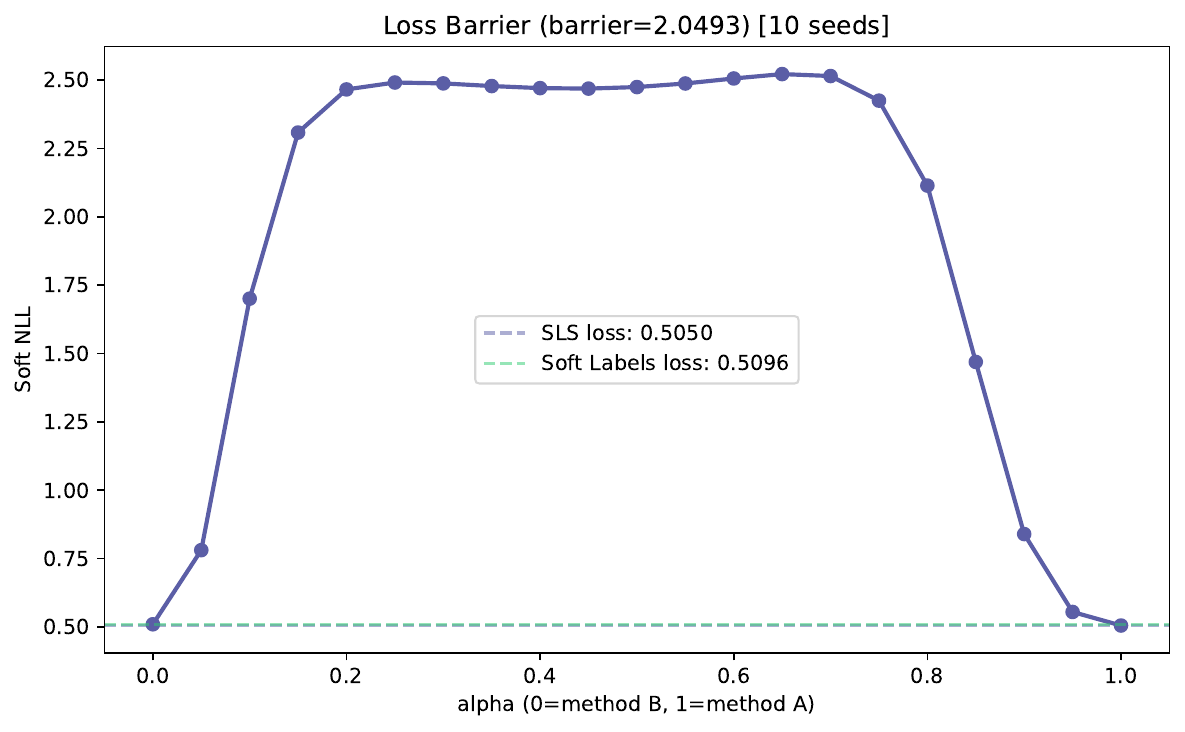}
\caption{Loss-barrier diagnostic between SLS and soft-label solutions. Across the 10 seed-paired interpolations, the mean barrier is 2.05, consistent with the two methods occupying different basins despite similar endpoint metrics.}
\label{fig:loss_barrier_appendix}
\end{figure*}

\begin{figure*}[!tp]
\centering
\includegraphics[width=0.8\textwidth]{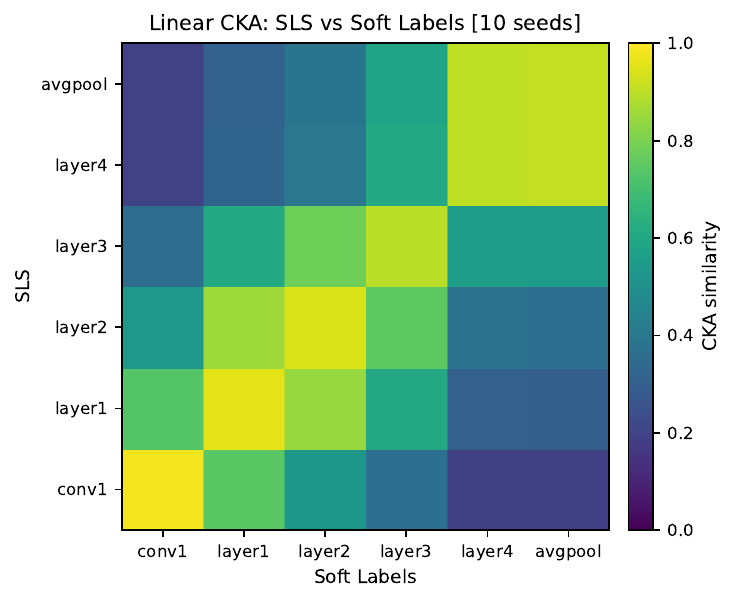}
\caption{Linear CKA representation similarity. Within-method penultimate-layer CKA is higher for SLS than for soft labels (0.920 vs.\ 0.887), indicating more reproducible representations under hard delivery.}
\label{fig:cka_appendix}
\end{figure*}

\subsection{OOD Results}
\label{app:ood}

\paragraph{SVHN family comparison.}
Table~\ref{tab:ood_family_svhn} reports the family-level SVHN far-OOD AUROC comparison. The table is best read as descriptive support rather than as a paired hypothesis test because the Soft labels column comes from the 10-seed main OOD outputs, whereas Multipass, Deterministic control, and SLS come from the 5-seed hard-delivery comparison.

\begin{table*}[!t]
\centering
\small
\caption{Family-level SVHN far-OOD AUROC. Higher is better. The Soft labels column comes from the 10-seed main OOD outputs, whereas Multipass, Deterministic control, and SLS come from the 5-seed hard-delivery comparison, so the comparison is descriptive rather than a paired hypothesis test. Hard-label delivery exceeds Soft labels on five of six detectors, and Deterministic control is highest on five of six.}
\label{tab:ood_family_svhn}
\resizebox{\textwidth}{!}{%
\begin{tabular}{lcccc}
\toprule
Score & Soft labels & Multipass & Deterministic control & SLS \\
\midrule
Energy         & 0.8634 $\pm$ 0.0429 & 0.8858 $\pm$ 0.0255 & \textbf{0.8890 $\pm$ 0.0253} & 0.8788 $\pm$ 0.0261 \\
MSP            & 0.8364 $\pm$ 0.0432 & 0.8553 $\pm$ 0.0182 & \textbf{0.8643 $\pm$ 0.0193} & 0.8520 $\pm$ 0.0183 \\
ODIN           & 0.8650 $\pm$ 0.0429 & 0.8871 $\pm$ 0.0242 & \textbf{0.8921 $\pm$ 0.0233} & 0.8812 $\pm$ 0.0254 \\
Pred.\ entropy & 0.8527 $\pm$ 0.0445 & 0.8714 $\pm$ 0.0212 & \textbf{0.8793 $\pm$ 0.0233} & 0.8660 $\pm$ 0.0208 \\
Logit margin   & 0.8112 $\pm$ 0.0436 & 0.8323 $\pm$ 0.0126 & \textbf{0.8419 $\pm$ 0.0136} & 0.8326 $\pm$ 0.0177 \\
KNN            & \textbf{0.7974 $\pm$ 0.0284} & 0.7836 $\pm$ 0.0250 & 0.7788 $\pm$ 0.0354 & 0.7741 $\pm$ 0.0295 \\
\bottomrule
\end{tabular}%
}
\end{table*}

\paragraph{Near-OOD results on CIFAR-100.}
Table~\ref{tab:ood_cifar100} reports the CIFAR-100 near-OOD AUROC comparison between SLS and soft labels across representative score types.

\begin{table}[!t]
\centering
\small
\caption{CIFAR-100 near-OOD AUROC for SLS versus soft labels. Higher is better.}
\label{tab:ood_cifar100}
\begin{tabular}{lccc}
\toprule
Score & SLS & Soft labels & $p$ \\
\midrule
Energy        & \textbf{0.8185} & 0.8096 & 0.0186 \\
ODIN          & \textbf{0.8189} & 0.8107 & 0.0186 \\
Pred.\ entropy & \textbf{0.8132} & 0.8082 & 0.0801 \\
MSP           & \textbf{0.8029} & 0.7972 & 0.0654 \\
KNN           & \textbf{0.8031} & 0.8009 & 0.3125 \\
Logit margin  & \textbf{0.7833} & 0.7779 & 0.0801 \\
\bottomrule
\end{tabular}
\end{table}

Table~\ref{tab:ood_family_cifar100} adds the family-level CIFAR-100 comparison across multipass, deterministic control, and SLS. As in the SVHN family comparison, the soft-label column comes from the 10-seed main OOD outputs whereas the hard-label family entries come from the 5-seed hard-delivery comparison.

\begin{table*}[!t]
\centering
\small
\caption{Family-level CIFAR-100 near-OOD AUROC across representative score types. Higher is better. The Soft labels column comes from the 10-seed main OOD outputs, whereas Multipass, Deterministic control, and SLS come from the 5-seed hard-delivery comparison.}
\label{tab:ood_family_cifar100}
\begin{tabular}{lcccc}
\toprule
Score & Soft labels & Multipass & Deterministic control & SLS \\
\midrule
Energy & 0.8096 & \textbf{0.8185} & 0.8183 & 0.8163 \\
MSP    & 0.7972 & \textbf{0.8032} & 0.8006 & 0.8003 \\
KNN    & 0.8009 & \textbf{0.8053} & 0.8050 & 0.8019 \\
\bottomrule
\end{tabular}
\end{table*}

\end{document}